\documentclass[letterpaper, 10 pt, conference]{ieeeconf}  

\IEEEoverridecommandlockouts                              

\usepackage{graphics} 
\usepackage{epsfig} 
\usepackage{mathptmx} 
\usepackage{times} 
\usepackage{algorithm}
\usepackage{algpseudocode}

\usepackage{amsthm}

\usepackage{amsmath} 
\usepackage{amssymb}  
\usepackage{nth}
\usepackage{hyperref}
\usepackage{subcaption}
\usepackage{color}

\title{\LARGE \bf
Constant Space Complexity Environment Representation for Vision-based Navigation
}

\author{Jeffrey Kane Johnson$^{1}$
\thanks{$^{1}$Jeffrey Kane Johnson received his PhD from Indiana University and is principal of Maeve Automation,
        Mountain View, CA 94043\newline
        \href{mailto:contact@maeveautomation.com}{\nolinkurl{contact@maeveautomation.com}}}%
}

\theoremstyle{plain}

\newtheorem{lemma}{Lemma}

\theoremstyle{definition}
\newtheorem{defn}{Definition} 
\newtheorem{problem}{Problem} 

\begin{document}

\maketitle
\thispagestyle{empty}
\pagestyle{empty}

\begin{abstract}
This paper presents a preliminary conceptual investigation into an environment representation that has  constant space complexity with respect to the camera image space. This type of representation allows the planning algorithms of a mobile agent to bypass what are often complex and noisy transformations between camera image space and Euclidean space. The approach is to compute per-pixel potential values directly from processed camera data, which results in a discrete potential field that has constant space complexity with respect to the image plane. This can enable planning and control algorithms, whose complexity often depends on the size of the environment representation, to be defined with constant run-time. This type of approach can be particularly useful for platforms with strict resource constraints, such as embedded and real-time systems.
\end{abstract}

\section{INTRODUCTION}

A significant issue in planning and control when solving real-world navigation problems is that there are often large numbers of individual agents with whom a mobile robot might interact. Consider navigating a busy roadway or crowded sidewalk or convention hall, where there may be multitudes of other agents sharing the space. Conventional approaches to planning in multi-agent systems often explicitly consider interactions between all agents and so become overwhelmed as the number of agents grows~\cite{DBLP:conf/itsc/BrechtelGD11,DBLP:journals/arobots/GalceranCEO17,DBLP:series/sbis/OliehoekA16}. More scalable conventional approaches often have strict requirements on system dynamics~\cite{DBLP:conf/isrr/BergGLM09} or observability of agent policies~\cite{DBLP:journals/ijrr/BekrisGMK12}.

This paper presents a preliminary conceptual investigation into the use of a fixed-size environment representation for vision-based navigation. The representation is modeled after a camera image space, which is chosen because cameras are a ubiquitous sensor modality, and image space is by nature discrete and fixed size. The proposed representation additionally allows planning and control routines to reason almost directly in sensor space thereby avoiding often complex and noisy transformations to and from a more conventional Euclidean space representation. The intent of this new representation is to help vision-based mobile robots navigate complex multi-agent systems efficiently, and to take a step toward satisfying the strict resource requirements often present in real-time, safety critical, and embedded systems~\cite{DBLP:conf/sigopsE/LouiseDDA02}.

The next section briefly surveys related work, then the environment representation is presented along with an illustrative example of how it can be used. Finally, conclusions and future work are discussed.

\section{RELATED WORK}
\begin{figure}
  \begin{subfigure}{\linewidth}
  \centering
  \includegraphics[width=0.95\linewidth]{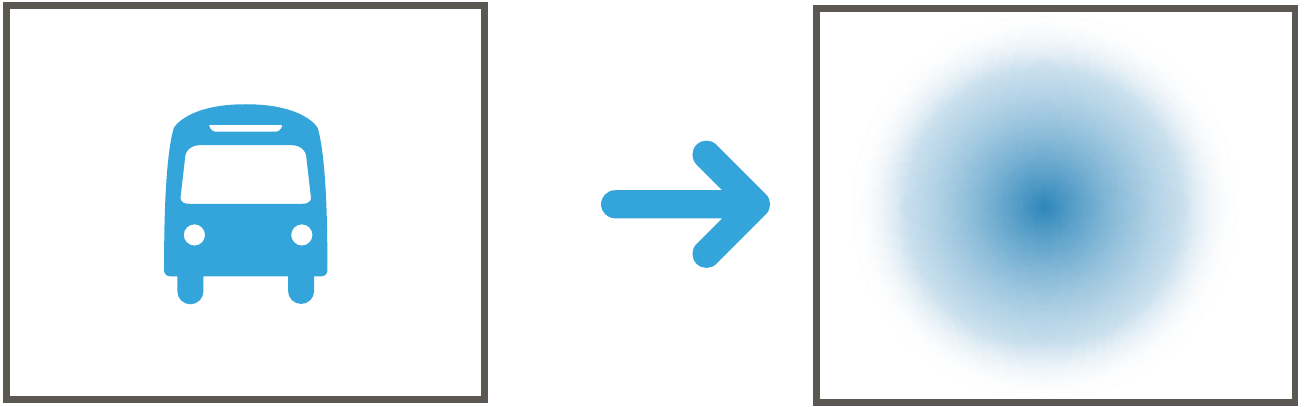}
  \end{subfigure}\par\medskip
  \begin{subfigure}{\linewidth}
  \centering
  \includegraphics[width=0.95\linewidth]{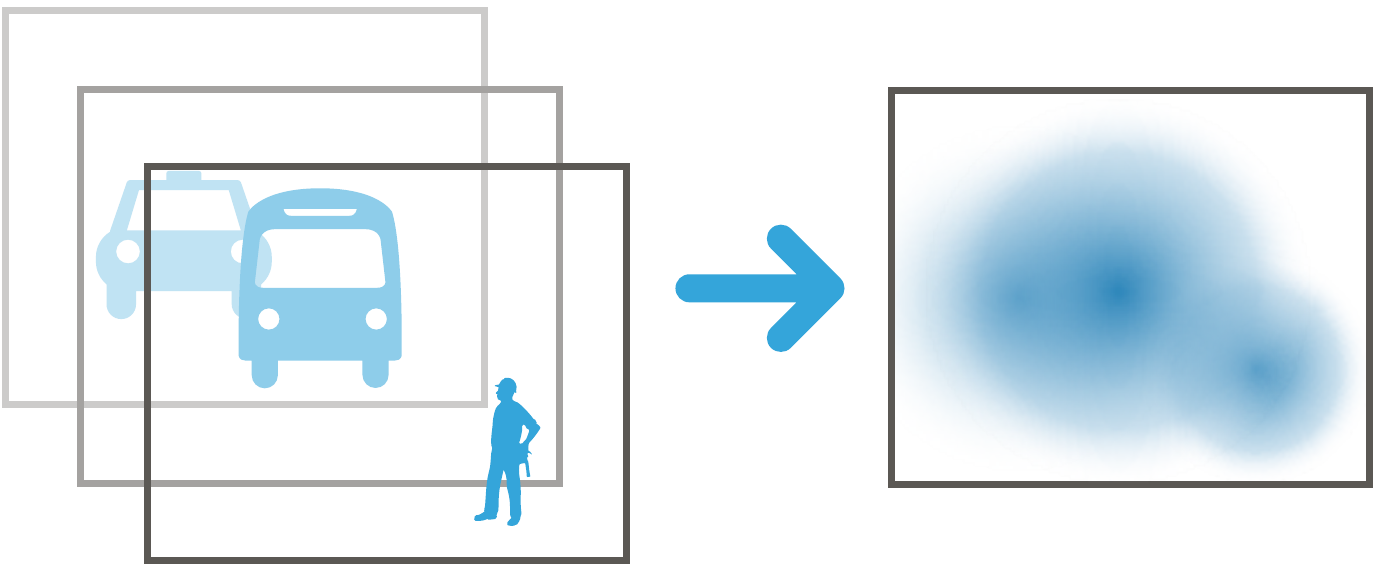}
  \end{subfigure}
\caption{Top: Illustration of an approaching object in the image plane (left), and the image space potential field (right). Bottom: Multiple object detections (left) can be composed into a single field (right). Black boxes represent ROIs.}\label{fig:image-space-fields}
\end{figure}

The approach in this work is based on {\em potential fields}~\cite{DBLP:conf/icra/Khatib85}. These fields represent attractive and repulsive forces as scalar fields over a robot's environment that, at any given point, define a force acting upon the robot that can be interpreted as a control command. This type of approach is subject to local minima and the {\em narrow corridor} problem~\cite{Koren91potentialfield}, particularly in complex, higher-dimensional spaces~\cite{DBLP:books/daglib/0016830}. Randomized approaches can partially overcome these difficulties~\cite{citeulike:1059835,DBLP:journals/ijrr/BarraquandL91}, and extensions to globally defined {\em navigation functions}~\cite{koditschek-aam-1990,1249705,DBLP:conf/icra/Koditschek87}, while often difficult to use in practice, theoretically solve them. This work uses potential fields defined over a virtual image plane, which limits the possibility of narrow corridors, and is designed such that additional information can be used by the controller to interpret potential gradients, as suggested in~\cite{DBLP:conf/icra/Masoud05}. \cite{DBLP:conf/icra/WolfB08} described a scheme related to that presented in this paper, but places potentials in a Euclidean space, which this approach explicitly avoids. The potential fields computed by the approach in this paper are intended to inform a {\em visual servoing}~\cite{Cowan:CSD-02-1200,DBLP:conf/icra/WeissSN85,DBLP:conf/icra/FomenaC07,DBLP:journals/corr/LeeLA17} control scheme. This is a class of controllers that computes control commands directly from image space data.

In order to define values for the potential fields, this approach draws on a wealth of related works in optical flow and monocular collision avoidance, notably~\cite{7321313,DBLP:conf/ecmr/AlenyaNC09,DBLP:conf/icra/ByrneT09,DBLP:conf/icra/MoriS13,DBLP:conf/icpr/CamusCHH96,DBLP:conf/isvc/PundlikPL11,DBLP:conf/iccv/CoombsHHN95,DBLP:conf/scia/Farneback03,DBLP:journals/pami/NelsonA89,NASA-TM-104025,DBLP:conf/icra/MoriS13,DBLP:conf/icra/ByrneT09}. The intuition of these approaches is that the information contained in a sequence of monocular stills provides sufficient information to compute {\em time-to-contact} (Definition~\ref{def:ttc}), which informs an agent about the rate of change of object proximity. The primary contribution of this work is a sensor-inspired representation space and algebra for enabling planning and control algorithms to reason effectively and efficiently with the output of this class of perception algorithms.

\section{IMAGE SPACE POTENTIAL FIELDS}
Before defining image space potential fields, the notion of a potential field used in this paper is defined below:

\begin{defn}\label{def:potential-field}
A {\em potential field} (also {\em artificial potential field}) is field of artificial forces that attracts toward desirable locations and repels from undesirable locations. In this work, a potential field is defined by a potential function that maps an image pixel value $I(x,y)$ to a tuple of affinely extended reals $\overline{\mathbb{R}}=\mathbb{R}\cup\{-\infty,+\infty\}$, the first of which is the potential value, and the second of which is its time derivative:
\begin{align}
I(x,y)\mapsto\overline{\mathbb{R}}^2
\end{align}
\end{defn}

From the definition, the image space potential (ISP) field is modeled after an image plane. As with image planes, the potential field is discretized into a grid, and regions of interest (ROIs) are defined for it. In this work it is assumed that the fields are origin and axis aligned with the camera images, and that they have the same ROIs (as in Figure~\ref{fig:image-space-fields}).

The potential function, which maps image pixel values to potential values, can be defined in arbitrary ways, either with geometric relations, learned relations, or even heuristic methods. In this paper geometric properties of temporal image sequences are used. The approach is to assume an estimation of the {\em time-to-contact} (defined below) is available for each pixel in an image over time. The mapping of image pixel to this value is taken as the potential function that defines the image space potential field.

\begin{defn}\label{def:ttc}
{\em Time-to-contact} ($\tau$), is the predicted duration of time remaining before an object observed by a camera will come into contact with the image plane of the camera. The time derivative of $\tau$ is written $\dot{\tau}$.
\end{defn}

As noted often in literature (e.g.~\cite{NASA-TM-104025,DBLP:conf/icpr/CamusCHH96,DBLP:conf/iccv/CoombsHHN95,DBLP:journals/pami/NelsonA89}), $\tau$ can be computed directly from the motion flow of a scene, which is defined as:

\begin{defn}
{\em Motion flow} is the pattern of motion in a scene due to relative motions between the scene and the camera. In other words, it is a vector field describing the motion of objects on the image plane over time.
\end{defn}

Unfortunately, it is typically not possible to measure motion flow directly, so it is usually estimated via {\em optical flow}, which is defined as the {\em apparent} motion flow in an image plane. Historically this has been measured by performing some kind of matching of, or minimization of differences between, pixel intensity values in subsequent image frames~\cite{DBLP:journals/ai/HornS81,DBLP:conf/scia/Farneback03,DBLP:conf/ijcai/LucasK81}, while more recently deep learning techniques have been successfully applied~\cite{DBLP:conf/iccv/WeinzaepfelRHS13}.

The image space potential field is now defined using $\tau$:

\begin{defn}\label{def:isp}
An {\em image space potential field} is defined by a potential function that maps image pixels to a tuple of scalar potential values $\langle\tau,\dot{\tau}\rangle$.
\end{defn}

\subsection{Computing $\tau$}\label{sec:computing-tau}
Assuming some reasonably accurate estimation of optical flow vector field exists, $\tau$ can be computed directly under certain assumptions~\cite{DBLP:conf/icpr/CamusCHH96}. In practice, the computation of optical flow tends to be noisy and error prone, so feature- and segmentation-based approaches can be used~\cite{DBLP:conf/icra/MoriS13,DBLP:conf/icra/ByrneT09}. The idea of these approaches is to compute $\tau$ from the rate of change in detection {\em scale}. For a point in time, let $s$ denote the scale (maximum extent) of an object in the image, and let $\dot{s}$ be its time derivative. When the observed face of the object is roughly parallel to the image plane, and under the assumption of constant velocity translational motion and zero yaw or pitch, it is straightforward to show that~\cite{camus-phdthesis}:

\begin{align}
\tau=\frac{s}{\dot{s}}\label{eq:tau-scale}
\end{align}

As shown by Lemma~\ref{lem:scale-invariance}, scale has a useful invariance property for these types of calculations that can make $\tau$ computations robust to certain types of detection noise:

\begin{lemma}\label{lem:scale-invariance}
The scale $s$ of an object on the image plane is invariant to transformations of the object under $SE(2)$ on the $XY$ plane.
\end{lemma}

\begin{proof}
Let $(X_1,Y_1,Z)$ and $(X_2,Y_2,Z)$ be end points of a line segment on the $XY$ plane in the world space, with $XY$ parallel to the image plane and $Z$ coincident with the camera view axis. Without loss of generality, assume unit focal length. The instantaneous scale $s$ of the line segment in the image plane is given by:

\begin{align}
s=\frac{1}{Z}\sqrt{\Delta X^2+\Delta Y^2}
\end{align}

Thus, any transformation of the line segment on the $XY$ plane for which $\Delta X^2+\Delta Y^2$ is constant makes $s$, and thereby $\dot{s}$ and $\tau$, independent of the values of $(X_1,Y_1)$ and $(X_2,Y_2)$. By definition, this set of transformations is $SE(2)$.
\end{proof}

In addition, as shown in~\cite{citeulike:2798975}, the time derivative $\dot{\tau}$ of $\tau$, when available, enables a convenient decision function for whether an agent's current rate of deceleration is adequate to avoid head-on collision or not. The decision function is given below, where $\varepsilon>0$ is a buffer to prevent an agent from coming to a stop directly at the point of contact with another agent:

\begin{align}
   f(\dot{\tau},\varepsilon) = \left\{
     \begin{array}{lr}
       1 & : \dot{\tau}\ge-0.5+\varepsilon\\
       0 & : \dot{\tau}<-0.5+\varepsilon
     \end{array}
   \right.\label{eq:tau-dot-decision}
\end{align}

Equation~\ref{eq:tau-scale} allows the computation $\tau$ for whole regions of the image plane at once given a time sequence of labeled image segmentations, and Equation~\ref{eq:tau-dot-decision} enables binary decisions to be made about the safeness of the agent's current state. The following two sections describe encoding these pieces of information into image space potential fields.

\subsection{Computing Fields for a Single Object}

Computing the image space potential field for a single object is straightforward given the discussion in \S\ref{sec:computing-tau}. Assuming an object can be accurately tracked and segmented over time in the camera image frame, its scale $s$ and estimated expansion $\dot{s}$ can be used to compute $\tau$ for each pixel in the image plane that belongs to the object, and a finite differences or estimation method can be used to compute $\dot{\tau}$. Pixels that do not belong to the object, and for which no other information is available, are mapped to $\langle\infty,\infty\rangle$ by the potential function. An illustration of this mapping is shown in the top row of Figure~\ref{fig:image-space-fields}.

\subsection{Computing Fields for Arbitrary Objects}

Computing the image space potential field for arbitrary objects builds on the single object case by computing the field individually for each segmented and tracked object and then composing them into a single field. For this composition to be meaningful, however, the fields cannot be simply added together; this would result in the destruction of the $\tau$ information. Instead, a composite field is defined to preserve and combine $\tau$ information meaningfully. Equation~\ref{eq:composed-field} defines a composite field $F$ in terms of image space potential fields $F_1$ and $F_2$ for an image $I$, and where $\min_\tau$ selects the tuple whose $\tau$ value is minimum:

\begin{align}
F(x,y)=\left\{\min_\tau\left(F_1(x,y),F_2(x,y)\right)\;|\;(x,y)\in I\right\}\label{eq:composed-field}
\end{align}

Selecting the point-wise $\tau$-minimum tuple for the composite field effectively enforces instantaneous temporal ordinality of objects, i.e., objects that are instantaneously temporally nearer are always what is seen. It is important to note that this is {\bf not} the same as spatial ordinality. For an illustration of this, see Figure~\ref{fig:ordinality}.

\subsection{Constant Space Complexity}

By definition the image space potential field representation has guaranteed constant space complexity assuming that the camera images for which the fields are generated are fixed size. This can be a powerful tool in simplifying planning and control algorithms whose complexity is typically dependent on the number of objects in a scene. In many cases it may, in fact, be possible to achieve constant time for planning and control given this representation.

It is important to note, however, that computing the representation itself may have arbitrary complexity: the problem of segmenting and tracking objects in order to generate these potential fields, for instance, can be efficient or arbitrarily complex, depending on the approach. The problem of investigating efficient computation methods for these fields is a point of future work discussed in \S\ref{sec:conclusion}. 

\section{NAVIGATION WITH IMAGE SPACE POTENTIAL FIELDS}
\begin{figure}
  \begin{subfigure}{\linewidth}
  \centering
  \includegraphics[width=0.95\linewidth]{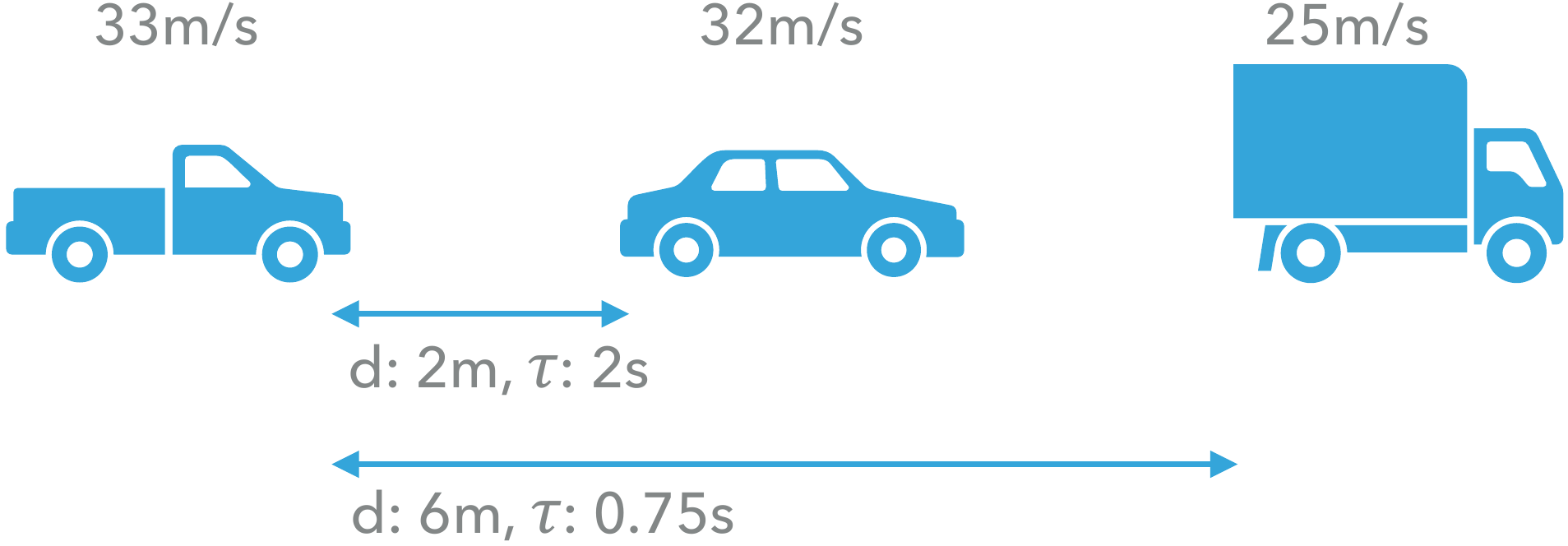}
  \end{subfigure}
\caption{Illustration comparing spatial and temporal ordinality. Consider three vehicles traveling in the same lane. For the pickup truck (left), the car (middle) has lowest spatial ordinality, i.e., is closest spatially. However, the van (right) has lowest temporal ordinality, i.e., it is nearest temporally.}\label{fig:ordinality}
\end{figure}

The intuition behind using image space potential functions for vision-based navigation is that they provide a natural space in which to compute collision avoiding controls, which then allows general navigation to be solved using a guided collision avoidance scheme, such as ORCA~\cite{DBLP:conf/isrr/BergGLM09} or the Selective Determinism framework%
~\cite{jkjohnson-phdthesis}. This section uses the Selective Determinism framework to define a simple control function that utilizes image space potential fields to navigate toward a visible goal in the presence of other agents. In this example only forward field of view is considered, but the extension to omnidirectional field of view is straightforward. The navigation problem considered is defined below:

\begin{problem}\label{prob:example-problem}
Assume a set of acceleration-bounded agents $\mathcal{A}$, each operating according to known dynamics and with similar capabilities, navigating a shared space. Each agent operates according to a unique reward function that is not observable to other agents. Each agent is equipped with a camera that always faces in the direction of motion, and each agent is capable of performing image segmentation on the output. Suppose the reward function for an agent $A$ encourages navigating toward a goal that $A$ has in line of sight. How can $A$ control toward the goal while avoiding collision with other agents?
\end{problem}

Problem~\ref{prob:example-problem} is the type of problem that a driver may face on a busy highway when trying to navigate toward an exit or offramp. The solution in this example will take a na\"{i}ve approach of decoupled steering and acceleration control while noting that more sophisticated control schemes are certainly possible. And while the example is formulated for a mobile agent traveling on a two dimensional manifold, the technique in general is equally applicable to three dimensions (such as with a UAV). The method for computing collision avoiding controls is discussed first, followed by the formulation of the navigation law.

\subsection{Collision Avoidance}

In order to address collision avoidance, the {\em Encroachment Detection} problem is presented.

\begin{problem}\label{prob:encroachment-detection}
Let {\em encroachment} refer to the reduction in minimum proximity between two or more objects in a workspace $\mathcal{W}$ beyond desired limits as measured by a metric $\mu(\cdot,\cdot)$. Assume an agent $A$ receives some observation input $O_t$ of $\mathcal{W}$ over time. Let $\mathcal{A}$ be the set of agents that does not include $A$. For a sequence of observations $O_i,\ldots,O_t$, how can $A$ estimate the rate of change of $\min_{A_j\in\mathcal{A}}\mu(A,A_j)$?
\end{problem}

Note that maintaining an estimate of $\langle\tau,\dot{\tau}\rangle$ directly solves the problem, as these values quantify temporal proximity and the rate of encroachment. The collision avoidance problem can now be solved by detecting encroachment and controlling such that it does not violate limits.

It was shown in \cite{Johnson-RSS-16} that collision avoidance can be guaranteed in a non-adversarial system if all agents compute and maintain collision-free stopping paths, which are contingency trajectories that bring an agent to zero relative velocity in minimal time. If agents are also self-preserving, they can each assume that all other agents will maintain such contingencies. Under these assumptions, agents should have sufficient information in the image space potential field to compute a solution to Problem~\ref{prob:encroachment-detection} by maintaining non-zero time headway, which is assumed to be witness to the existence of a feasible stopping path.

For illustration, a na\"{i}ve control set computation in the spirit of the braking controller in~\cite{Johnson_semiautonomouslongitudinal,6338731} is sketched in Algorithm~\ref{alg:virtual-bumpers}.
This routine makes the reasonable assumption that $\dot{\tau}$ is not so large as to overwhelm $\tau$. The idea is that steering angle and acceleration commands are computed independently and such that $\tau$ thresholds are not violated. To compute steering angles, a min filter is swept across the field of view in the potential field and a minimum potential value within the window is computed for each column in the image. Any value that meets $\tau$ thresholds is kept, and these are considered the safe steering angles (Figure~\ref{fig:control}, left). To compute the acceleration command, the minimum potential value within a centered window of a specified width is considered (Figure~\ref{fig:control}, right). If the value meets the $\tau$ threshold, the full scaled range of accelerations, $[-1,1]$, is considered safe. If the threshold is violated, then either full deceleration $[-1,-1]$ or the range of deceleration values $[-1,0)$ is sent depending on the value of the decision function of Equation~\ref{eq:tau-dot-decision}. The control sets are then used by the Selective Determinism framework to compute the output control command.

\begin{algorithm}
\caption{Given an image space potential field $F$, compute the set of steering and acceleration commands that satisfy $\tau\ge T_s$ and $\dot{\tau}\ge-0.5+\varepsilon$, where $T_s>0$ is some desired time headway, $w_\theta$ and $w_a$ are kernel widths for computing steering angle and acceleration maps, and $\varepsilon>0$ is the buffer from Equation~\ref{eq:tau-dot-decision}.}
\label{alg:virtual-bumpers}
\begin{algorithmic}[1]
\Procedure{SafeControls}{$F,T_s,\dot{\tau}_E,w_\theta,w_a,\varepsilon$}
\State Let $I_c$ be the list of image column indices
\State Let $M_a$ map $i\in I_c$ to steering angles
\State Let $h$ be the height (row count) of $F$
\State Let $M_\tau$ map $\langle\tau,\dot{\tau}\rangle$ to $i\in I_c$ via $w_\theta\times h$ min filter\label{line:theta-window}
\State Let $M_\theta=\{\langle\tau,\dot{\tau}\rangle\in M_\tau\;:\;\tau\ge T_s\}$
\State Let $W$ be a centered $w_a\times h$ window in $F$\label{line:a-window}
\State Let $\langle\tau,\dot{\tau}\rangle_\textrm{min}$ be the min. $\tau$ over $W$
\State Let $L\gets\emptyset$ be a container for safe accelerations
\If{$M_\theta=\emptyset$}
\State $M_\theta\gets0\;,\;L\gets[-1,-1]$
\ElsIf{$\tau_\textrm{min}>T_s$}
\State $L\gets[-1,1]$
\Else
\If{$f(\dot{\tau},\varepsilon)=0$}
\State $L\gets[-1,-1]$
\Else
\State $L\gets[-1,0)$
\EndIf
\EndIf
\State\Return$M_\theta$, $L$
\EndProcedure
\end{algorithmic}
\end{algorithm}

\subsection{The Selective Determinism Framework}
\begin{figure}
  \begin{subfigure}{\linewidth}
  \centering
  \includegraphics[width=0.95\linewidth]{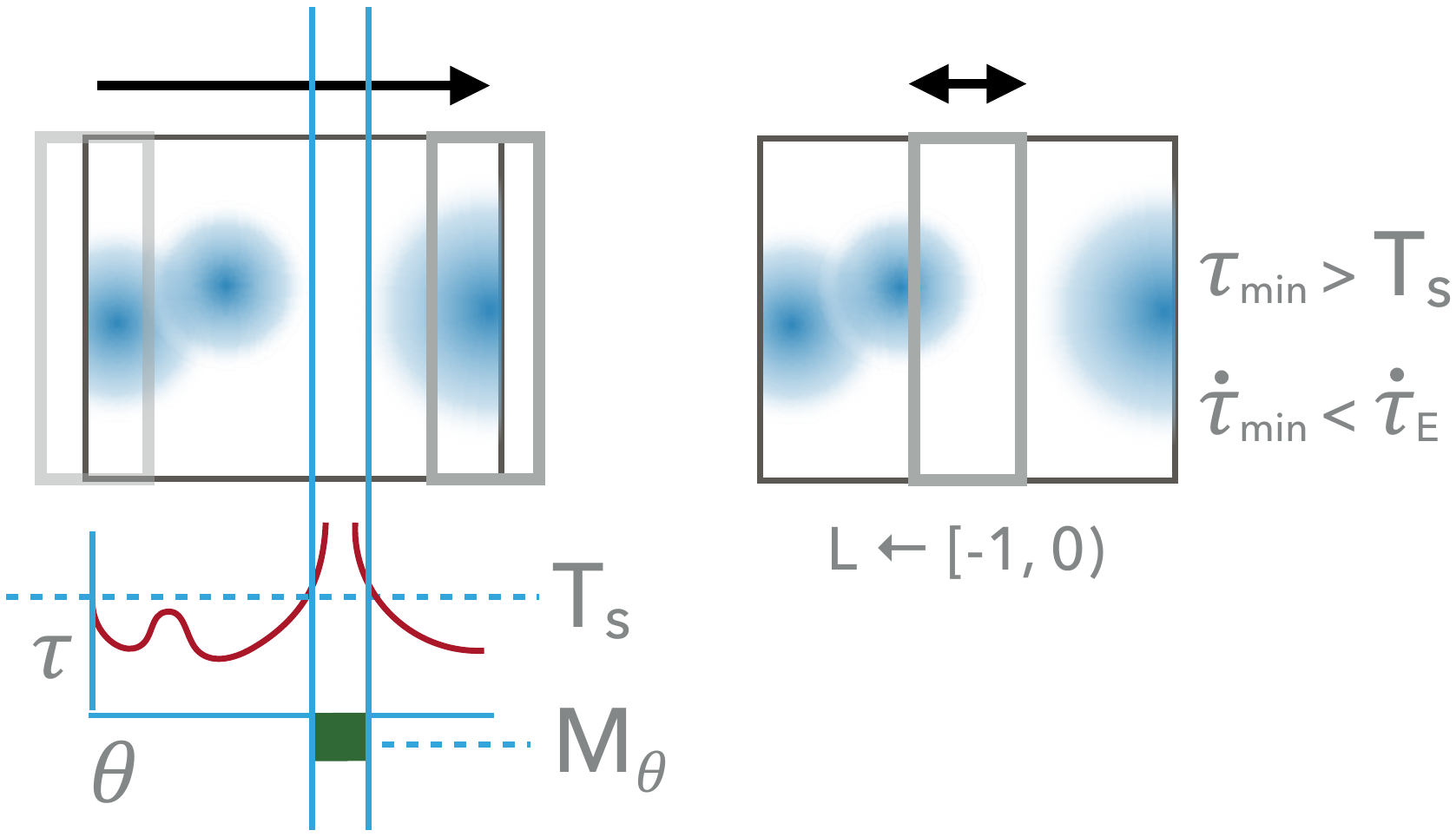}
  \end{subfigure}
\caption{Illustration of the steering angle control computation (left) and the acceleration control computation (right). On the left, a window sweeps from left to right over the image space potential field computing minimum $\tau$ for each image space column (left bottom). The set of column values that satisfy the threshold are the set of acceptable steering angles $M_\theta$. On the right, the minimum potential value over a centered window is computed and the set $L$ of acceptable scaled acceleration values are determined from it.}\label{fig:control}
\end{figure}

\begin{figure}
  \begin{subfigure}{\linewidth}
  \centering
  \includegraphics[width=0.95\linewidth]{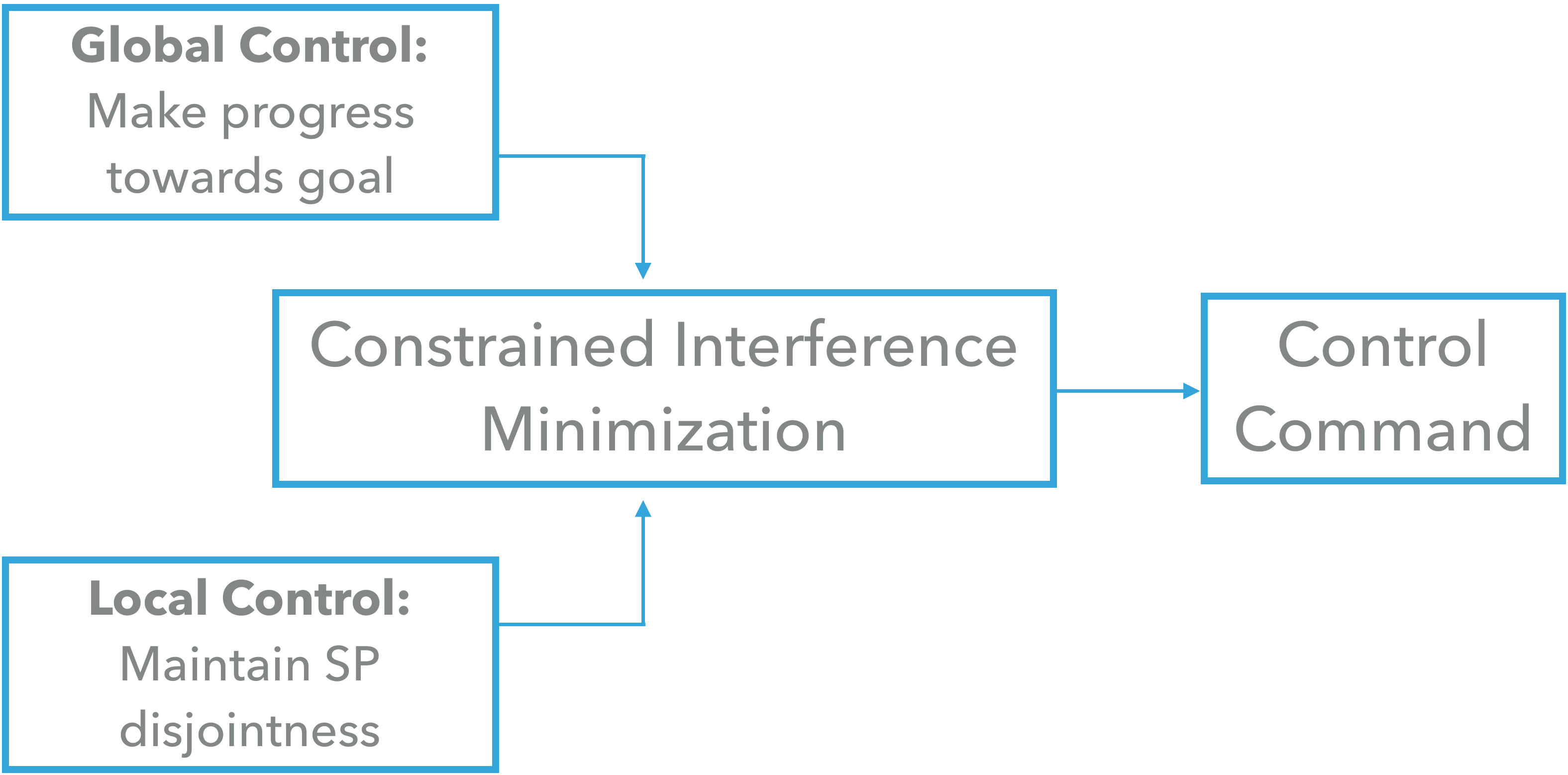}
  \end{subfigure}
\caption{Architecture of the Selective Determinism framework.}\label{fig:sd}
\end{figure}

Selective Determinism~\cite{jkjohnson-phdthesis} is a solution framework for dynamically-constrained, non-adversarial, partially-observable multi-agent navigation problems. It belongs to a family of approaches useful for dealing with real-world problems because they remove the theoretical intractability inherent in optimal approaches~\cite{Weinstein:2013:OPL:2891460.2891661,YuNgNOMDP} while typically exhibiting good empirical performance. Selective Determinism, in addition, can also make certain collision avoidance guarantees even without explicitly considering interaction effects among agents.

Selective Determinism works by exploiting the idea that agents in a system are capable of independently computing contingency trajectories that cover the space necessary for them to come to a stop (or to a zero relative velocity), and it assumes that agents do so, and that they will seek to maintain a non-empty set of available contingencies.

The framework casts the navigation problem in terms of a constrained interference minimization~\cite{Johnson_semiautonomouslongitudinal,6338731} that utilizes a local collision avoidance controller to compute sets of controls from which an optimization choses a control that makes maximal progress toward some goal (see Figure~\ref{fig:sd}). The solution to Problem~\ref{prob:example-problem} is sketched in Algorithm~\ref{alg:virtual-bumpers-sd}.

\begin{algorithm}
\caption{For a desired pixel location $(x_d,y_d)$, and setpoint speed $\dot{s}_d$, compute the Selective Determinism control that safely guides the agent $A$ toward $(x_d,y_d)$. See Algorithm~\ref{alg:virtual-bumpers} for descriptions of the other parameters.}
\label{alg:virtual-bumpers-sd}
\begin{algorithmic}[1]
\Procedure{Controls}{$(x_d,y_d),F,T_s,\dot{\tau}_E,w_\theta,w_a,\varepsilon$}
\State Let $\theta_t,\dot{s}_t$ be the steering angle and speed of $A$
\State Let $\theta_d$ be the steering angle corresponding to $y_d$
\State Let $M_\theta,L\gets\textrm{SafeControls}(F,T_s,\dot{\tau}_E,w_\theta,w_a,\varepsilon)$\label{line:safe-control}
\State Let $\theta^\star\gets\theta_t$ contain the new steering angle
\For{$\theta\in M_\theta$}
\If{$|\theta-\theta_d|<|\theta^\star-\theta_d|$}
\State $\theta^\star\gets\theta$
\EndIf
\EndFor
\State Let $\ddot{s}^\star\in L$ be chosen proportionally to $\dot{s}_d-\dot{s}_t$\label{line:comp}
\State\Return $\theta^\star,\;\ddot{s}^\star$
\EndProcedure
\end{algorithmic}
\end{algorithm}

\subsection{Complexity Analysis}

In Algorithm~\ref{alg:virtual-bumpers} all non-trivial operations are iterations over the width of the image plane, which is assumed to be fixed for a given problem. The operations on lines~\ref{line:theta-window} \&~\ref{line:a-window} depend on the user defined $w_\theta$ and $w_a$ parameters, but these are also bounded by image width. In Algorithm~\ref{alg:virtual-bumpers-sd}, Line~\ref{line:safe-control} is a call to Algorithm~\ref{alg:virtual-bumpers}, and so has constant complexity with respect to the image space, and Line~\ref{line:comp} is assumed to be implemented with an $O(C)$ proportional law. Thus, the navigation algorithm as a whole has constant complexity, in space and time, with respect to the camera image space.

\section{CONCLUSION \& FUTURE WORK}\label{sec:conclusion}

This paper presented a conceptual investigation into an environment representation for vision-based navigation that has constant space complexity with respect to the image. This preliminary work is intended to serve as a basis for future investigations. This section outlines three primary topics of investigation.

The first topic is how to more completely combine environment information with the potential fields. As presented here, the representation is defined strictly in terms of object $\tau$ values, but a more elegant solution would build richer information about the environment into the potential field itself. An obvious extension would be to encode path information, such as lane or sidewalk boundaries, as well as goal information, into the potential field. This would enable navigation in semantically sophisticated environments and is an area of active development\footnote{\href{https://bitbucket.org/maeveautomation/maeve\_automation\_core/src}{https://bitbucket.org/maeveautomation/maeve\_automation\_core/src}}.

The second topic is a more sophisticated control law. The decoupled approach used here can lead to odd and counter-productive behavior, such as swerving out of the way of approaching objects while at the same time slowing down, or instability around control modes. A more intelligent control law would address stability issues and reason about the steering and longitudinal controls simultaneously. Additionally, allowing the potential fields to label a small, fixed-size set of objects individually could let such a control law reason about individual interactions without losing constant space complexity or information about all other objects.

Finally, the third topic, and one of great importance, is whether and how the potential fields themselves can be computed with some kind of constant complexity. A purely optical flow based approach would address this, but would require breakthroughs in the quality and efficiency of optical flow algorithms. Alternatively, a purely learning-based approach in conjunction with cheap, heuristic-based tracking approaches may provide the requisite segmentation and tracking information without runaway complexity.

\addtolength{\textheight}{-2cm}   



\bibliographystyle{IEEEtran}
\bibliography{IEEEabrv,refs}

\end{document}